\documentclass{article}

\usepackage[dblblindworkshop,nonatbib, preprint]{neurips_ACA_2025}

\workshoptitle{\\ Workshop on Algorithmic Collective Action (ACA@NeurIPS 2025)}

\usepackage[utf8]{inputenc} 
\usepackage[T1]{fontenc}    
\usepackage{hyperref}       
\usepackage{url}            
\usepackage{booktabs}       
\usepackage{amsfonts}       
\usepackage{nicefrac}       
\usepackage{microtype}      
\usepackage{xcolor}         
\usepackage{amsmath}
\usepackage{amsthm}
\usepackage{thmtools, thm-restate}
\usepackage{enumitem}
\usepackage{bm}
\usepackage{hyperref}

\newtheorem{lemma}{Lemma}

\newcommand{\Sctr}{S_{\text{ctr}}}
\newcommand{\Sal}{S_{\text{al}}}

\title{Algorithmic Collective Action \\ with Multiple Collectives}

\author{%
  Claudio Battiloro$^*$\\
  Harvard University \\
  \And
 Pietro Greiner\thanks{Corresponding authors. Emails: \texttt{ cbattiloro@hsph.harvard.edu,pietro.greiner@mila.quebec}} \\
  Mila, LawZero\\
  \And
  Bret Nestor\\
  Harvard University\\
  \And
  Oumaima Amezgar\\
  University of Padova \\
  \And
  Francesca Dominici\\
  Harvard University \\
}

\begin{document}

\maketitle

\begin{abstract}
As learning systems increasingly influence everyday decisions, user-side steering via Algorithmic Collective Action (ACA)—coordinated changes to shared data— offers a complement to regulator-side policy and firm-side model design. Although real-world actions have been traditionally decentralized and fragmented into multiple collectives despite sharing overarching objectives-with each collective differing in size, strategy, and actionable goals, most of the ACA literature focused on single collective settings. In this work, we present the first theoretical framework for ACA with multiple collectives acting on the same system. In particular, we focus on collective action in classification, studying how multiple collectives can plant signals, i.e., bias a classifier to learn an association between an altered version of the features and a chosen, possibly overlapping, set of target classes. We provide quantitative results about the role and the interplay of collectives' sizes and their alignment of goals. Our framework, by also complementing previous empirical results, opens a path for a holistic treatment of ACA with multiple collectives.
\end{abstract}

\section{Introduction}\label{sec:intro}
AI’s rapid spread rides on massive training datasets, yielding stronger predictors, and wider applications. However, as data-driven systems shape and guide more aspects of life they raise acute risks: privacy violations, leakage of
sensitive information, and biased decisions that entrench inequality.

Responses span firms, regulators, and users. At the firm's level, “Trustworthy AI” programs embed fairness checks, bias mitigation, privacy audits, and red teaming across the lifecycle \cite{barocas2023fairness}, but often clash with performance and engagement goals. At the regulators' level, privacy laws—e.g., GDPR \cite{EU2016GDPR}, PIPEDA \cite{Canada2000PIPEDA}, and CPRA \cite{California2020CPRA}—set floors, yet compliance alone rarely ensures socially responsible outcomes \cite{selbst2019fairness,utz2019informed}. Finally, at the user's level, Algorithmic Collective Action (ACA) has recently emerged as an appealing paradigm to empower grassroots efforts by organizing users to coordinate data contributions or refusals to steer models \cite{HardtEtAl2023,devrio2024building}.

This work is motivated by the fact that the ACA literature has almost entirely focused on scenarios where only a single collective acts. However, real-world actions have traditionally been carried out by multiple collectives, resulting in decentralized and fragmented practices. Different collectives often vary widely in size, strategy, and even actionable goals, yet they can still align on common overarching objectives (e.g. climate justice or gender equality). Bruno Latour’s actor-network theory \cite{Latour2005Reassembling} conceptualized why such movements rarely behave as a single unit: what is regarded as a global campaign is really a web of local interactions with no singular center. These loose networks manage to act together only through continual coordination. Indeed, they are inherently dynamic and not intrinsically unified. An interesting example is the cyberfeminist struggle, which has never been one monolithic front; instead, it consists of distributed efforts ranging from feminist hacker-art collectives to global hashtag campaigns. Donna Haraway’s Cyborg Manifesto \cite{Haraway1991CyborgManifesto} famously proposed the cyborg as a metaphor for coalition across difference, emphasizing affinity over rigid identity categories. In that spirit, cyberfeminism has taken many decentralized forms around the world while still sharing core aims of challenging patriarchal tech culture and gender inequality.  Environmental activism provides another strong example: the climate justice movement is a patchwork of groups employing diverse tactics (from youth-led social media strikes to militant direct-action cells) to achieve various actionable goals (from permit denials to pipeline stoppages and enforcement actions), but they all strive toward the same overarching objective of halting climate breakdown. Here, we aim to reflect and model the need for considering multiple collectives in ACA.


\textbf{Related Works.} ACA as users steering ML outputs toward a group goal was recently rigorously formalized in \cite{HardtEtAl2023}, building on the Data Leverage framework’s view of data as an active instrument rather than a passive input \cite{VincentEtAl2020}. In \cite{SiggEtAl2024}, the authors proposed a combinatorial model for ACA to study the strategic interaction between drivers and delivery platforms, inspired by the \#DeclineNow DoorDash campaign. The work in \cite{BenDovEtAl2024} studied how the firm's chosen learning algorithms affect the success of ACA.  Motivated by the growing regulatory focus on privacy and data protection, \cite{solanki2025crowding} studied the effectiveness of ACA if AI firms train differentially private models. The work in \cite{baumann2024algorithmic} focused on ACA in recommender systems. The authors of \cite{gauthier2025statistical} proposed an alternative theoretical framework to \cite{HardtEtAl2023} that empowers collectives via statistical inference, enabling them to learn better ACA strategies and infer the parameters that determine their success. The closest work to ours is \cite{karan2025algorithmic}, in which a conceptual framework for ACA with multiple collectives is introduced, and experiments are performed. However, no theoretical treatment is provided, and most of the focus is on ACA with two collectives. Our work is inspired by the empirical findings of \cite{karan2025algorithmic} and, partly, by its conceptual framing. We move further by designing the first rigorous treatment for ACA with multiple collectives, generalizing the original theoretical framework from \cite{HardtEtAl2023}.

\textbf{Contribution.} We introduce the first theoretical framework for ACA with multiple collectives. We focus on collective action in classification, studying how $M$ collectives can plant signals, i.e., bias a classifier to learn an association between an altered version of the features $x$ and a chosen, possibly overlapping, set of target classes $\{y_c\}_{c=1}^{S\leq M}$. We analyze two regimes: when collectives can act on both features and labels, and when collectives are limited to acting on features only. For each regime, we derive lower bounds on the per-collective and global success, revealing interesting trade-offs driven by the interplay of the sizes of the collectives and by how closely their actionable goals align. We illustrate the framework with a use case in climate adaptation.  Finally, we discuss what future methodological and empirical directions are needed to enrich our approach and make it a holistic framework for ACA with multiple collectives.

\section{Algorithmic Collective Action with Multiple Collectives}\label{sec:MACA}
We start from the definition of ACA with a single collective as given in \cite{HardtEtAl2023,gauthier2025statistical}. A firm deploys a learning algorithm on data coming from a certain user population. Within the overall user population, a collective forms that represents a given share of users (a positive fraction, but not the whole). This collective wants to steer the platform’s outcome toward a certain objective. To do so, the collective agrees on one strategy for editing their contributions—potentially changing features, labels, or both—before those contributions reach the platform. The platform then receives a training set that blends unmodified data from the baseline process with the collective’s edited data, and it deploys its algorithm on this blended data. The collective’s impact is evaluated by a tailored success measure. Given its size, the collective’s problem is to choose the editing rule that maximizes this success once the platform has trained on the blended data. Here, we rigorously generalize ACA \cite{HardtEtAl2023} to a setting with multiple collectives, each acting with a possibly different strategy and actionable goal, but attempting to maximize both per-collective and global measures of success. 

\textbf{Data space.} We define the data space $Z = X \times Y$ as the product of two finite spaces $X$ and $Y$, being the feature and label spaces, respectively. All probability measures on $Z$ will be defined on the discrete $\sigma$-algebra $2^Z$ of~$Z$.

\textbf{Population.} We define the the population as a probability space $(\Omega, 2^{\Omega}, \pi)$, with $|\Omega| = N$ finite. Let $f_0: \Omega \to Z$ be a random variable. We define the population distribution as the push-forward measure $P_0 :=(f_0)_{\#}\pi$. Intuitively, $\Omega$ is the user population, with each element of $\Omega$ being a user, and $\pi$ is the sampling distribution. Then, $f_0$ assigns a feature-label pair to each user, and $P_0$ is the induced distribution over feature–label pairs, i.e., it records, for each $(x,y)\in Z$, the $\pi$-fraction of users whose assigned pair is $(x,y)$.

\textbf{Collectives, Strategies, and Masses.}
We define an ensemble of $M$ collectives as $n$ disjoint subsets $\Omega_1, \dots, \Omega_M$  of $\Omega$ such that $\pi(\Omega_c)=\alpha_c > 0$. We refer to $\alpha_c$ as the mass of the $c$-th collective. If $\pi$ is the normalizing counting measure and $|\Omega_c|= N_c$, then $\alpha_c$ is exactly the fraction of the population $N_c/N$ belonging to the collective $\Omega_c$.  Each collective $\Omega_c$ agrees on a potentially randomized strategy $h_c: Z\rightarrow Z$ from a space of available strategies, i.e., feasible changes to the data. We define the total mass as $ \bar\alpha := \sum_{c=1}^{M}\alpha_{c}$.  For every collective $\Omega_c$ and every $z\in Z$, we assume $P_0(z)=P_0(z \mid \omega \in \Omega_c)$, i.e., before applying the strategy $h_c$, every collective $\Omega_c$ is just a random slice of the same underlying population. 

\textbf{Mixture Distribution.} In this setting, the action of a collective $\Omega_c$ of applying its strategy $h_c$ steers $P_0$ towards $P_c := (h_c)_{\#}P_0$ with strength $\alpha_c$. As such, a firm aiming to train a model $m: X\rightarrow Y$ on data coming from $\Omega$ will observe the mixture distribution:
\begin{equation}\label{eq:mixture}
    P(\{\alpha_c\}_c, \{h_c\}_c)
    \;=\;
    \bigl(1-\bar\alpha\bigr)\,P_{0}
    \;+\;
    \sum_{c=1}^{M}\alpha_{c}P_{c}.
\end{equation}
\textbf{Per-collective and Global Success.} Given a strategy $h_c$ for each collective, we denote by $S_c(\alpha_c)$ and $S(\{S_c\}_c)$ the measure of per-collective and global success, respectively. In this work, we study how per-collective success $S_c(\alpha_c)$ is maximized as a function of the mass $\alpha_c$, and how global success $S(\{S_c\}_c)$ is maximized as a function of the per-collective successes $\{S_c\}_c$.

\section{ACA with Multiple Collectives in Classification}\label{sec:maca_class}
Following \cite{HardtEtAl2023,gauthier2025statistical}, we start by studying the case where the firm's learning algorithm $m$ is a classifier. We assume that the firm chooses an approximately optimal classifier on the distribution $P$.

\textbf{$\varepsilon$-suboptimal Classifier.}\cite{HardtEtAl2023} A classifier $m: X \rightarrow Y$ is $\varepsilon$-suboptimal on a set $X^{\prime} \subseteq X$ under the distribution $P$ if there exists a $P^{\prime}$ with $\operatorname{TV}\left(P_{Y \mid X=x}, P_{Y \mid X=x}^{\prime}\right) \leq \varepsilon$ such that for all $x \in X^{\prime}$
\begin{equation}
m(x)=\underset{y \in Y}{\operatorname{argmax}}\; P^{\prime}(y \mid x).
\end{equation}
Since any classifier is at worst $0.5$-suboptimal, we assume $\varepsilon< 0.5$ without loss of generality.

\textbf{Planting Signals.} We focus on the setting in which each collective $\Omega_c$ wants the classifier $m$ to learn an association between an altered version of the features $g_c(x)$ and a chosen target class $y_c^*$. Formally, given a transformation $g_c: X \rightarrow X$ induced by the strategy $h_c$, each collective $\Omega_c$ measures per-collective success as
\begin{equation}
S_c(\alpha_c)=\operatorname{Pr}_{x \sim P_0}\left[m(g_c(x))=y_c^*\right].
\end{equation}
As explained in the Sec.\ref{sec:MACA} , the ensemble of $M$ collective then adopts and aims to maximize a measure of global success $S(\{S_c\}_c)$. In this work, we propose and study two different choices for $S(\{S_c\}_c)$:
\begin{description}[style=unboxed,leftmargin=0pt,labelsep=0.5em]
\item[\textbf{(No One is Left Behind)}] We first consider an egalitarian metric, which takes the per-collective success of the least successful collective as the global measure os success (for example, worst-group accuracy, group DRO~\cite{Sagawa2020Distributionally}), i.e.,
\begin{equation}\label{eq:global_min}
         S_{\min}(S_c)
        := \min_{c\in [M]} \; S_{c}(\alpha_c).
    \end{equation}
 \item[\textbf{(The Bigger the Better)}]  We then consider a metric proportional to each collective's mass, which averages the per-collective successes of all the collectives, weighting them with their masses, i.e.,
\begin{equation}\label{eq:global_avg}
        S_{\mathrm{w}}(S_c)
        := \frac{1}{\bar\alpha}\sum_{c=1}^{M}\alpha_{c}\,S_{c}(\alpha_c).
    \end{equation}
\end{description}
Generalizing the single collective case \cite{HardtEtAl2023}, we investigate simple strategies for planting signals. We then study their success as a function of four key sets of parameters: the collectives' masses $\{\alpha_c\}_c$, the per-collective uniqueness $\{\xi_c\}_c$ of each signal they aim to plant, the global uniqueness $\xi$ of the signals they aim to plant, and their target alignments $\{\beta_{c}\}_{c}$. Let us define the marginal distributions $P_0^X$ and $P_0^Y$ of $P_0$ as  $P_0^X:=\left(\gamma_X\right)_{\#} P_0$ and $P_0^Y:=\left(\gamma_Y\right)_{\#} P_0$, where $\gamma_X: X \times Y \rightarrow X$ and $\gamma_Y: X \times Y \rightarrow X$ are the projections $\gamma_X(x, y)=x$ and $\gamma_Y(x, y)=y$. If there is no risk of ambiguity, we denote $P_0^X(x)$ and $P_0^Y(y)$ simply with $P_0(x)$ and $P_0(y)$, respectively. Similarly, we will employ the same shorthand notation for any other involved marginal distribution when possible. Furthermore, let us denote with $X^*\subseteq X$ the subset of the feature space representing the image of the elegible set under $h_c$, i.e., the image of the features alterable by the adoption of $h_c$.  

\textbf{Per-collective $\xi_c$-Unique Signals.}\cite{HardtEtAl2023} For each collective $\Omega_c$,  we say that a signal of the collective $\Omega_c$ is individually $\xi_c$-unique if it satisfies $P_0\left(X_c^*\right) \leq \xi_c$. Intuitively, $\xi_c$ is a measure of the rarity of the signal $\Omega_c$ is trying to plant.

\textbf{Global $\xi$-Unique Signals.} For each collective $\Omega_c$, the marginal $P_c^X$ models where $\Omega_c$ deposits feature mass after its editing action $g_c$. We say signals are globally $\xi$-unique if, for all $c$ and all $j \neq c$,
\begin{equation}
P_{j}^X(X_c^*)\leq \xi.
\end{equation}
Intuitively, the signals are globally $\xi$-unique if it is rare that collectives emit the same feature point under their "actioned" distribution  $P_c^X$.

\textbf{Target Alignment.} For each collective $\Omega_c$, we define  the target alignment as
\begin{equation}
\beta_{c}:=\sum_{\substack{j \neq c:  y_j^*=y_c^*}} \alpha_j.
\end{equation}
Therefore, $\beta_{c}$ measures the total mass of all the collectives that have the same target as $\Omega_c$.

\textbf{Suboptimality Gap.}\cite{HardtEtAl2023} For each collective $\Omega_c$, we define the suboptimality gap of $y_c^*$ as
\begin{equation}
\Delta_c=\max _{x \in X_c^*}\left(\max _{y \in Y} P_0(y \mid x)-P_0\left(y_c^* \mid x\right)\right)
\end{equation}
Intuitively, $\Delta_c$ measures how suboptimal the target label $y^*_c$ is on the signal set under the base distribution $P_0$.

As usually done \cite{HardtEtAl2023,gauthier2025statistical,solanki2025crowding}, we consider (i) the case in which the users can modify both features and labels, referring to the resulting strategies as feature-label strategies, and (ii) the case in which the users can access both features and labels, but modify only the features, referring to the resulting strategies as feature-only strategies.

\textbf{Feature-label Signal Strategy.} For each collective $\Omega_c$, the feature-label signal strategy is given by
\begin{equation}\label{eq:FL_strat}
h_c(x, y)=\left(g_c(x), y_c^*\right).
\end{equation}
In this case, $X^* = g_c(X)$.

\noindent We derive the following lower bound on the per-collective success of each collective as a function of the key parameters described above.
\begin{restatable}{theorem}{FLBound}
\label{FLThm}
For each collective $\Omega_c$, under the feature-label strategy in \eqref{eq:FL_strat}, $\varepsilon$-suboptimality of the classifier $m$, and $\left(\xi_c, \xi \right)$-uniqueness, it holds:
\begin{equation}\label{eq:bound_FL}
  S_c\left(\alpha_c\right) \geq 1\;-\;\xi_c \cdot \frac{\Delta_c + 2\varepsilon}{1-2\varepsilon} \cdot \frac{1-\bar{\alpha}}{\alpha_c}\;-\;\xi \cdot \frac{1+2\varepsilon}{1-2 \varepsilon} \cdot \frac{\bar{\alpha}-\alpha_c-\beta_c}{\alpha_c}.
\end{equation}
\end{restatable}

\textbf{Proof.} See Appendix \ref{Proofs}.

\textbf{Feature-only Signal Strategy.} For each collective $\Omega_c$, the feature-only signal strategy is given by
\begin{equation}\label{eq:F_strat}
h_c(x, y)= \begin{cases}\left(g_c(x), y_c^*\right), & \text { if } y=y_c^* \\ (x, y), & \text { otherwise }\end{cases}.
\end{equation}
In this case, $X_c^* = g_c\!\left(\mathrm{supp}\,P_0(\cdot\mid y=y_c^*)\right)$.

\noindent Again, we derive a lower bound on the per-collective success of each collective as a function of the key parameters. For a collective $\Omega_c$, feature-only strategies can fail when the base distribution strongly favors some label $y \neq y^*_c$, leaving little label uncertainty. To avoid this degenerate case, we make the positivity assumption  $P_0(y_c^*\mid x)\ge p_c >0$ for all $x\in X$ \cite{HardtEtAl2023}.
\begin{restatable}{theorem}{FBound}\label{FThm}
For each collective $\Omega_c$, under the feature-only strategy in \eqref{eq:F_strat}, $\varepsilon$-suboptimality of the classifier $m$, $\left(\xi_c, \xi \right)$-uniqueness, and $p_c \leq (1+2 \varepsilon) / 2$, it holds:
\begin{equation}\label{eq:bound_F}
    S_c(\alpha_c) \ \ge\ 1\;-\;
\xi_c\,\frac{(1-\bar{\alpha})(1+2\varepsilon-2p_c)\;+\;2\varepsilon\,\bar{\alpha}}{(1-2\varepsilon)\,p_c} \cdot \frac{1}{\alpha_c}
\;-\;
\xi\,\frac{1+2\varepsilon}{1-2\varepsilon}\,\frac{\bar{\alpha}-\alpha_c-\beta_c}{\alpha_c\,p_c}.
\end{equation}
\end{restatable}

\textbf{Proof.} See Appendix \ref{Proofs}.

\textbf{Discussion.} The bounds in \eqref{eq:bound_FL} and \eqref{eq:bound_F} are the complement of the sum of two products. The first product can be interpreted as a “rarity × difficulty × mass”. In particular, is scales linearly with: 
\begin{description}[style=unboxed,leftmargin=0pt,labelsep=0.5em]
    \item[(\textbf{Rarity})] How rare the signal of the collective is under the baseline, $\xi_c$;
\item[(\textbf{Difficulty})] How hard it is to flip the baseline on that signal, $\frac{\Delta_c + 2\varepsilon}{1-2\varepsilon}$ and  $\frac{(1-\bar{\alpha})(1+2\varepsilon-2p_c)\;+\;2\varepsilon\,\bar{\alpha}}{(1-2\varepsilon)} $;
    \item[(\textbf{Mass})] How much effective mass the collective brings, $\frac{1-\bar\alpha}{\alpha_c}$ and $\frac{1}{\alpha_c\,p_c}$. 
    \end{description}
    The second product can be interpreted as a "competition" penalty given by “signal overlap ×  misalignment”. In particular, it scales linearly with:
    \begin{description}[style=unboxed,leftmargin=0pt,labelsep=0.5em]
        \item[(\textbf{Signal Overlap})]  How often others land on the collective's signal, $\xi$;
        \item[(\textbf{Misalignment})]  The fraction of collectives not aligned with the target $y_c^*$, $\frac{\bar{\alpha}-\alpha_c-\beta_c}{\alpha_c}$ and $\frac{\bar{\alpha}-\alpha_c-\beta_c}{\alpha_c\,p_c}$.
    \end{description}
We now derive lower bounds for the global success measures in \eqref{eq:global_min}-\eqref{eq:global_avg} under both feature-label and feature-only strategies, as direct consequences of Theorems \ref{FLThm}-\ref{FThm}, respectively.
\begin{restatable}{corollary}{GlobalFLBound}
For each collective $\Omega_c$, under the feature-label strategy in \eqref{eq:FL_strat}, $\varepsilon$-suboptimality of $m$, and $(\xi_c,\xi)$-uniqueness, it holds:
\begin{gather}
    S_{\min}\!\left(S_c\right)
    \ge
    1-\max_{c \in [M]}
    \Bigg[
      \xi_c \cdot \frac{\Delta_c+2\varepsilon}{1-2 \varepsilon} \cdot \frac{1-\bar{\alpha}}{\alpha_c}
      \;+\;
      \xi \cdot \frac{1+2 \varepsilon}{1-2 \varepsilon} \cdot \frac{\bar{\alpha}-\alpha_c-\beta_c}{\alpha_c}
    \Bigg], \label{eq:glob_bound_FL_min}\\[4pt]
    S_{\operatorname{w}}\!\left(S_c\right)
    \ge
    1
    -\frac{1-\bar{\alpha}}{\bar{\alpha}}
      \sum_{c=1}^{M} \xi_c \cdot \frac{\Delta_c+2\varepsilon}{\alpha_c (1-2 \varepsilon)}
    -\xi\cdot\frac{1+2 \varepsilon}{1-2 \varepsilon} \cdot \frac{1}{\bar{\alpha}}
      \sum_{c=1}^{M} \frac{\bar{\alpha}-\alpha_c-\beta_c}{\alpha_c}.
      \label{eq:glob_bound_FL_avg}
\end{gather}
\end{restatable}
\begin{restatable}{corollary}{GlobalFBound}
For each collective $\Omega_c$, under the feature-only strategy in \eqref{eq:F_strat}, $\varepsilon$-suboptimality of $m$, $(\xi_c,\xi)$-uniqueness, and the positivity assumption, it holds:
\begin{gather}
S_{\min}\!\left(S_c\right)
\;\ge\;
1-\max_{c\in[M]}
\Bigg[\;
\xi_c\,\frac{(1-\bar{\alpha})(1+2\varepsilon-2p_c)\;+\;2\varepsilon\,\bar{\alpha}}{(1-2\varepsilon)\,p_c} \cdot \frac{1}{\alpha_c}
\;+
\xi\,\frac{1+2\varepsilon}{1-2\varepsilon}\,\frac{\bar{\alpha}-\alpha_c-\beta_c}{\alpha_c\,p_c}
\Bigg],\label{eq:glob_bound_F_min}\\[4pt]
S_{\mathrm{w}}\!\left(S_c\right)
\;\ge\;
1
-\frac{1}{\bar{\alpha}}
\sum_{c=1}^{M}
\Bigg[\xi_c\,\frac{(1-\bar{\alpha})(1+2\varepsilon-2p_c)\;+\;2\varepsilon\,\bar{\alpha}}{(1-2\varepsilon)\,p_c} \cdot \frac{1}{\alpha_c}
\;+
\xi\,\frac{1+2\varepsilon}{1-2\varepsilon}\,\frac{\bar{\alpha}-\alpha_c-\beta_c}{\alpha_c\,p_c}
\Bigg].\label{eq:glob_bound_F_avg}
\end{gather}
\end{restatable}
\textbf{On Critical Masses and Alignment.} A core quantity of ACA with a single collective is the critical mass, i.e., the smallest mass required to ensure a target level $S^*$ of success for the collective. However, as one could intuitively expect and as our theoretical results confirm, if multiple collectives are present, then not only their masses but also their alignment on target classes are especially important to determine the success of an action. As such, in ACA with multiple collectives, two key questions can be investigated to derive conditions about when a target global success threshold $S^*$ is met: (i) "Given a certain alignment, how big should the masses be to achieve $S^*$?, and (ii) "Given certain masses, how big should the alignment be to achieve $S^*$?. These questions naturally suggest leveraging the bounds in \eqref{eq:glob_bound_FL_min}-\eqref{eq:glob_bound_FL_avg}-\eqref{eq:glob_bound_F_min}-\eqref{eq:glob_bound_F_avg} to derive lower bounds on collectives' masses (given a certain alignment) and alignment (given certain masses), for a fixed $S^*$. We will explore this key aspect in an extended version of this work.

\section{A Real-World Use Case in Climate Adaptation}\label{sec:example}
In this section, we provide a key real-world use case for ACA with multiple collectives in the space of interventions for climate adaptation. Consider a city that uses a text classifier trained on neighborhood-level records to identify necessary interventions. The neighborhood-level records encompass both official statistics and community-driven forum discussions. In this case, if $V$ is the vocabulary of the text classifier, then $X=2^V$. The model recommends one intervention per neighborhood from a finite menu $Y$ (e.g., street trees and shade, cool roofs or pavement, expanded cooling centers, or targeted heat-health outreach). Assume to have $M$ collectives, being grassroots neighborhood associations. Each collective is run by neighborhood residents and therefore has a nuanced, ground-level understanding of local conditions—including qualitative, hard-to-measure factors—and seeks to coordinately change its related shared data to steer the classifier toward the intervention $y^*_c \in Y$ it has judged to be most necessary. Each collective thus deploys an editing strategy $h_c$ to achieve its actionable goal. For example, in a feature–label regime, a collective of a waterfront neighborhood documents chronic flooding with household infiltration tests and evidence collected by interviewing community members. The collective then edits its features and labels using that documentation to bias the forum discussions in a coordinated way and explicitly request rain gardens/bioswales. Similarly, collectives would likely need to resort to feature-only strategies if the city’s intake pipeline does not allow community label edits, i.e., if the intervention field is locked and labeled by staff/administrative codes. Citywide (a.k.a. global) success would then emphasize equity (lifting the worst-served neighborhood with $S_{\min}$) or scale (a mass-weighted average with $S_{\operatorname{w}}$). 
Even if this example is presented in a didactic way, there is precedence for community-reported concerns to be ingested into decision-making tools. NOAA– and CAPA-led "Heat Watch" campaigns have produced neighborhood-scale heat maps in \(\,120+\) communities and feed them into planning and public-health practice \cite{noaa_heatwatch,capastrategies_heatwatch}. European cities are deploying city‐scale digital twins: Rotterdam’s Open Urban Platform with a Digital Twin underpins climate-resilience analysis and participatory planning, with citywide availability announced from January 2025 \cite{rotterdam_digital_twin_vision}.

\section{Future Directions} 

The theoretical framework we designed is just the first step to comprehensively characterize ACA with multiple collectives. Several research directions can be pursued.

\noindent On the theoretical and methodological side, as we mention in Sec. \ref{sec:MACA}, it is needed to properly charcaterize how the notion of critical mass from \cite{HardtEtAl2023} evolves in the multiple collectives setting. Moreover, it would be interesting to generalize the statistical framework of \cite{gauthier2025statistical} to a multi-collective setting. Another possibility is studying scenarios where collectives seek to erase \cite{HardtEtAl2023,gauthier2025statistical} or unplant \cite{gauthier2025statistical} signals, as well as mixed-objective settings in which some collectives plant signals while others erase or unplant them. Similarly, a relevant case is also when collectives have heterogeneous capabilities, thus some collectives can use feature–label strategies while others are restricted to feature-only strategies.

\noindent On the applied side, it is crucial to design other meaningful use-cases, beyond the one proposed in Sec. \ref{sec:example}.  Moreover, it is necessary to perform exhaustive experiments to validate the proposed bounds and study the impact of the proposed strategies in real-world scenarios. Finally, it would be interesting to formally reconcile our theoretical framework with the conceptual characterization presented in \cite{karan2025algorithmic}, fostering mutual enrichment.

\section{Conclusion}\label{sec:conclusion}
In this work, we introduced the first theoretical framework for Algorithmic Collective Action with multiple collectives. We focused on collective action in classification, studying how multiple collectives can bias a classifier to learn an association between an altered version of their features and a chosen, possibly overlapping, set of target classes. We provided quantitative results showcasing the interesting interplay of collectives' sizes and their alignment of values. Finally, we discussed a potential use case for our framework in the space of interventions for climate adaptation, validating the importance of taking into account the actions of multiple collectives and their interaction.

\bibliographystyle{plainurl}
\bibliography{ref}
\appendix
\section{Proofs}
\label{Proofs}
\begin{lemma}
\label{eml}
If $m$ is $\varepsilon$-suboptimal on a set $X$, then for any $x \in X$:
\[
  P(y^* \mid x) \ge \max_{y \neq y^*} P(y \mid x) + 2\varepsilon
  \quad \implies \quad m(x)=y^* .
\]
\end{lemma}

\begin{proof}
By $\varepsilon$-suboptimality on $X$, for every $x \in X$ there exists a conditional distribution $P'(\cdot \mid x)$ on $Y$ such that
\[
m(x) \in \arg\max_{y\in Y} P'(y\mid x)
\qquad\text{and}\qquad
\frac12 \sum_{y\in Y}\big|P'(y\mid x)-P(y\mid x)\big| \le \varepsilon .
\]
For every $y\in Y$, it then holds
\[
P'(y^*\mid x)\ \ge\ P(y^*\mid x)-\varepsilon
\qquad\text{and}\qquad
P'(y\mid x)\ \le\ P(y\mid x)+\varepsilon .
\]
Fix $x\in X$ and assume
\(
P(y^*\mid x) \ge \max_{y\ne y^*} P(y\mid x) + 2\varepsilon.
\)
Then, for every $y\ne y^*$,
\[
P'(y^*\mid x)\ \ge\ P(y^*\mid x)-\varepsilon
\ \ge\ \max_{z\ne y^*} P(z\mid x)+\varepsilon
\ \ge\ P(y\mid x)+\varepsilon
\ \ge\ P'(y\mid x).
\]
Hence $y^*$ is an (in fact, the) maximizer of $P'(\cdot\mid x)$, so $m(x)=y^*$.
\end{proof}

\begin{lemma}
\label{cfl}
For $a,b>0$, it holds that
\[
  1_{\{a<b\}} \;\le\; \frac{b}{a}.
\]
\end{lemma}
\begin{proof}
If $a<b$, then $1_{\{a<b\}}=1$ and $b/a>1$, so $1\le b/a$.  
If $a\ge b$, then $1_{\{a<b\}}=0$ and $0\le b/a$ since $a>0$.  
In both cases, the inequality holds.
\end{proof}

\FLBound*
\begin{proof}
    We are going to prove the theorem in two steps: first we will provide a point wise estimate for $x\in X$ and then we will proceed to bound the expectation over $X$.
    
    \textit{First step:} Fix a collective $\Omega_c$ and $x \in X$. For every $y \in Y$ we define the unnormalized score $N(y \mid x)$:
    \[
        N(y \mid x) := (1- \bar{\alpha}) P_0(y \mid x) P_0(x) + \sum_{j=1}^M \alpha_j P_j(x,y).
    \]
    Under the feature label strategy we have that $P_c(x,y)=1_{\lbrace Y=y_c^* \rbrace}P_c(x)$. So for any $y \neq y_c^*$ we have that:
    \[
         N(y_c^* \mid x) -  N(y \mid x) =  (1- \bar{\alpha}) (P_0(y_c^* \mid x) - P_0(y \mid x)) P_0(x) + \underbrace{\sum_{j:y_j^*=y_c^*} \alpha_j P_j(x)}_{S_{al} }- \underbrace{\sum_{j:y_j^*=y} \alpha_j P_j(x)}_{S_{y}}.
    \]
    In this context we can apply lemma \ref{eml} as follows: if for every $y \neq y^*$ it holds that
    \[
    \label{Mdiseq1}
    \tag{$A_1$}
        N(y^* \mid x) -  N(y \mid x) \geq 2\varepsilon\,P(x)
    \]
    where
    \[
        P(x)= (1- \bar{\alpha})P_0(x) + \sum_{j=1}^M \alpha_j P_j(x)
    \]
    then $m(x)=y^*$.

    Now in order to get \ref{Mdiseq1} we find a lower bound for the left-hand side using the following two observations:
    \begin{itemize}
        \item for $x \in X_c^*$, by definition of $\Delta_c$
            \[
                P_0(y^*_c \mid x) - P_0(y \mid x) \geq - \Delta_c
            \]
            hence:
            \[
                (1- \bar{\alpha}) (P_0(y_c^* \mid x) - P_0(y \mid x)) P_0(x) \geq - (1- \bar{\alpha}) \Delta_c P_0(x);
            \]
        \item for every $y \neq y^*_c$:
        \[
            \Sctr := \sum_{j:y_j^*\neq y^*_c} \alpha_j P_j(x) \geq \sum_{j:y_j^*=y} \alpha_j P_j(x) = S_y.
        \]
    \end{itemize}
    With these two observations, we can lower-bound the left-hand term of \ref{Mdiseq1}:
    \[
        N(y^*_c \mid x) -  N(y \mid x) \geq -(1-\bar{\alpha})\Delta_c\, P_0(x) + \Sal- \Sctr
    \]
    which is independent from $y$. We also use $\Sctr$ to rewrite the right hand term as:
    \[
        2\varepsilon\,P(x) = 2\varepsilon\big((1- \bar{\alpha})P_0(x) +\Sal +\Sctr\big).
    \]
    So we now have the following sufficient condition:
    \[
        -(1-\bar{\alpha})\Delta_c P_0(x) + \Sal- \Sctr \geq 2\varepsilon\big((1- \bar{\alpha})P_0(x) +\Sal +\Sctr\big).
    \]
    We can now rearrange to get:
    \[
        (1-2\varepsilon)\,\Sal \;-\; (1+2\varepsilon)\,\Sctr \ \geq\ (1- \bar{\alpha})\big( \Delta_c+2\varepsilon\big)\,P_0(x),
    \]
    and since $\Sal \geq \alpha_c P_c(x)$, a stricter sufficient condition is
    \[
        \alpha_c P_c^X(x) \ge (1-\bar{\alpha}) \underbrace{\frac{\Delta_c + 2\varepsilon}{1 - 2\varepsilon}}_{:=A_c(\varepsilon,\Delta_c)} P_0(x) \;+\; {\underbrace{\frac{1+2\varepsilon}{1-2\varepsilon}}_{=:C(\varepsilon)}}\,\Sctr(x) \ =: R(x).
    \]
    This sufficient condition can be expressed in terms of characteristic functions as:
    \[
    \label{cf_inequality}
    \tag{A2}
        1_{\lbrace m(x)=y^* \rbrace} \geq 1_{\lbrace\alpha_c P_c^X(x) \geq R(x)\rbrace }.
    \]

    \textit{Second step:} First we apply lemma \ref{cfl} to the right hand side of \ref{cf_inequality} we get:
    \[
         1_{\lbrace m(x) \neq y^* \rbrace} \leq 1_{\lbrace\alpha_c P_c^X(x) < R(x) \rbrace } \leq \frac{R(x)}{\alpha_c P_c^X(x)}.
    \]
    now we can take the expectation over $x \sim P_c(x)$:
    \[
        1-S_c = \mathbb{E}_{P_c(x)} \big[1_{\lbrace m(x) \neq y^* \rbrace}\big] \leq \frac{1}{\alpha_c} \mathbb{E}_{P_c(x)} \left\lbrack(1-\bar{\alpha})A_c\frac{P_0(x)}{P_c(x)} + C(\varepsilon)\sum _{j:y^*_j \neq y^*_c} \alpha_j\frac{{P_j(x)}}{P_c(x)}\right\rbrack.
    \]
    Now we use the following facts:
    \begin{itemize}
        \item Rarity identity (FL): $\mathbb{E}_{P_c^X}\left[\frac{P_0(x)}{P_c^X(x)}\right] = P_0(X_c^*) \le \xi_c.$
        \item Overlap control: $\mathbb{E}_{P_c^X}\left[\frac{P_j^X(x)}{P_c^X(x)}\right] =\sum_{x \in X^*_c}P_j^X(x) \le \xi \text{ for each } j \neq c.$
    \end{itemize}
    And since
    \[
        \sum _{j:y^*_j \neq y^*_c} \alpha_j = \bar{\alpha}- \alpha_c - \beta_c
    \]
    we get:
    \[
      S_c\left(\alpha_c\right) \geq 1\;-\;\xi_c \cdot \frac{\Delta_c + 2\varepsilon}{1-2\varepsilon} \cdot \frac{1-\bar{\alpha}}{\alpha_c}\;-\;\xi \cdot \frac{1+2\varepsilon}{1-2 \varepsilon} \cdot \frac{\bar{\alpha}-\alpha_c-\beta_c}{\alpha_c}. \qedhere
    \]
    
\end{proof}

\FBound*
\begin{proof}
We follow similar steps as the proof of \ref{FLThm}. 

\textit{First step:}
Fix a collective $c$ and $x\in X$. For every $y\in Y$ define
\[
N(y\mid x)\ :=\ (1-\bar{\alpha})\,P_0(y\mid x)\,P_0(x)\;+\;\sum_{j=1}^M \alpha_j\,P_j(x,y).
\]
Under the \emph{feature-only} strategy, each $j$ modifies only the slice $y_j^*$ and leaves all other pairs $(x,y\ne y_j^*)$ unchanged. Writing
\[
Q_j^X \ :=\ (g_j)_{\#}P_0(\,\cdot \mid y=y_j^*\,),
\]
we have
\[
P_j(x,y)\;=\;1_{\{y=y_j^*\}}\,P_0(y_j^*)\,Q_j^X(x)\;+\;1_{\{y\neq y_j^*\}}\,P_0(x,y).
\]
Hence, for any $y\neq y_c^*$,
\begin{align}
N(y_c^*\mid x)-N(y\mid x)
&\ge (1-\bar{\alpha})\big(P_0(y_c^*\mid x)-P_0(y\mid x)\big)P_0(x) \nonumber \\
&\hspace{2em}+\underbrace{\sum_{j:\,y_j^*=y_c^*}\alpha_j\,P_0(y_c^*)\,Q_j^X(x)}_{S_{\mathrm{al}}(x)}
-\underbrace{\sum_{j:\,y_j^*=y}\alpha_j\,P_0(y)\,Q_j^X(x)}_{S_{y}(x)}.\label{eq:bound_diff_n}
\end{align}
Moreover,
\begin{gather}
P(x)=(1-\bar{\alpha})P_0(x)+\sum_{j=1}^M \alpha_j P_j^X(x),\\
P_j^X(x)= \sum_{y \neq y_j^*} P_0(x, y)+P_0\left(y_j^*\right) Q_j^X(x)=P_0(x)-P_0(x,y_j^*)+P_0(y_j^*)\,Q_j^X(x). \label{eq:identity_Pj}
\end{gather}
By Lemma~\ref{eml}, a sufficient condition for $m(x)=y_c^*$ is
\[
\tag{$A3$}\label{Mdiseq2}
N(y_c^*\mid x)-N(y\mid x)\ \ge\ 2\varepsilon\,P(x)\qquad\forall\,y\neq y_c^*.
\]
We bound the two sides of \eqref{Mdiseq2} pointwise, using:
\begin{itemize}
\item Uniform positivity: For the maximizer $y\neq y_c^*$ of $P_0(y\mid x)$, $P_0(y_c^*\mid x)\ge p_c$ implies
\(
P_0(y_c^*\mid x)-P_0(y\mid x)\ge 2p_c-1.
\)

\item Contradictors: $S_{y}(x)\le \Sctr(x)$ for all $y\neq y_c^*$, where
\[
\Sctr(x)\ :=\ \sum_{j:\,y_j^*\neq y_c^*}\alpha_j\,P_j^X(x).
\]

\item Aligned mass: Keep only the self-aligned contribution:
\(
S_{\mathrm{al}}(x)\ \ge\ \alpha_c\,P_0(y_c^*)\,Q_c^X(x).
\)

\item Mixture upper bound:
Using \eqref{eq:identity_Pj},
we can write
\begin{align*}
P(x)=\ P_0(x)\ -\ \sum_{j}\alpha_j\,P_0(x,y_j^*)\ +\ \sum_{j}\alpha_j\,P_0(y_j^*)\,Q_j^X(x)\leq\ P_0(x)\ +\ \sum_{j}\alpha_j\,P_0(y_j^*)\,Q_j^X(x).
\end{align*}
Remembering that $P_0(y_j^*)\,Q_j^X(x)\ \le\ P_j^X(x)$ and splitting the last sum into aligned and contradicting indices, we obtain
\[
P(x)\ \le
\ \ P_0(x)\ +\ S_{\mathrm{al}}(x)\ +\ \Sctr(x).
\]
Hence
\(
2\varepsilon\,P(x)\ \le\ 2\varepsilon\big(P_0(x)+S_{\mathrm{al}}(x)+\Sctr(x)\big).
\)
\end{itemize}

Tightening the sufficient condition \eqref{Mdiseq2} by setting the bound in \eqref{eq:bound_diff_n} $\geq 2\epsilon P(x)$, combining this with the four bullets above,  and rearranging, we obtain
\begin{equation}\tag{$A4$}\label{eq:A4}
(1-2\varepsilon)\,\alpha_c\,P_0(y_c^*)\,Q_c^X(x)\;-\;(1+2\varepsilon)\,\Sctr(x)
\ \ge\ \Big((1-\bar{\alpha})(1+2\varepsilon-2p_c)+2\varepsilon\,\bar{\alpha}\Big)\,P_0(x),
\qquad x\in X_c^*.
\end{equation}

For $x\in X_c^*$ define the preimage ratios
\[
r_c(x)\ :=\ \frac{P_0(x)}{P_0\!\big(g_c^{-1}(x)\big)}\,,\qquad
t_c(x)\ :=\ \frac{\Sctr(x)}{P_0\!\big(g_c^{-1}(x)\big)}\,.
\]
Using $P_0(y_c^*)\,Q_c^X(x)=P_0\!\big(g_c^{-1}(x),y_c^*\big)$ and uniform positivity,
\[
\frac{P_0\!\big(g_c^{-1}(x),y_c^*\big)}{P_0\!\big(g_c^{-1}(x)\big)}
=\frac{\sum_{x^{\prime} \in g_c^{-1}(x)} P_0\!\left(y_c^* \mid x^{\prime}\right) P_0\!\left(x^{\prime}\right)}{P_0\!\left(g_c^{-1}(x)\right)}
\ \ge\ \frac{1}{P_0\!\left(g_c^{-1}(x)\right)} \sum_{x^{\prime} \in g_c^{-1}(x)} p_c\, P_0\!\left(x^{\prime}\right)
=\ p_c.
\]
Dividing \eqref{eq:A4} by $P_0(g_c^{-1}(x))$ and rearranging then yields the \emph{preimage} sufficient condition
\begin{equation}\tag{$A5$}\label{eq:A5}
\alpha_c\,p_c\ \ge\
(1-\bar{\alpha})\,\underbrace{\frac{1+2\varepsilon-2p_c}{1-2\varepsilon}}_{=:B_{y_c^*}(\varepsilon,p_c)}\,r_c(x)\;+\;
\frac{2\varepsilon\,\bar{\alpha}}{1-2\varepsilon}\,r_c(x)\;+\;
\underbrace{\frac{1+2\varepsilon}{1-2\varepsilon}}_{=:C(\varepsilon)}\,t_c(x),\qquad x\in X_c^*.
\end{equation}

\textit{Second step:} Define the post–action feature law as the pushforward
$\lambda_c \ :=\ (g_c)_{\#} P_0^X$ of $P_0^X$ through $g_c$.
With this notation, the per–collective success can be written as the expectation under $\lambda_c$ (law of the unconscious statistician):
\[
S_c(\alpha_c)
\ =\ \Pr_{x\sim P_0^X}\!\big[m(g_c(x))=y_c^*\big]
\ =\ \sum_{x\in X} 1\{m(x)=y_c^*\}\,\lambda_c(x).
\]

By Lemma~\ref{cfl} and the preimage margin condition \eqref{eq:A5} proved in the first step, we have for all $x\in X$,
\[
1\{m(x)\neq y_c^*\}\ \le\ 1\{x\in X_c^*\}\,
\frac{(1-\bar{\alpha})\,B_{y_c^*}(\varepsilon,p_c)\,r_c(x)\;+\;\displaystyle\frac{2\varepsilon\,\bar{\alpha}}{1-2\varepsilon}\,r_c(x)\;+\;C(\varepsilon)\,t_c(x)}{\alpha_c\,p_c}\,.
\]
Summing over $x\in X$ and using that $\mathrm{supp}(\lambda_c)\subseteq X_c^*$, we obtain 
\[
\begin{aligned}
1-S_c(\alpha_c)
&= \sum_{x\in X}1\{m(x)\neq y_c^*\}\,\lambda_c(x)\\[2pt]
&\le \frac{(1-\bar{\alpha})\,B_{y_c^*}(\varepsilon,p_c)}{\alpha_c\,p_c}
\sum_{x\in X_c^*}\! r_c(x)\,\lambda_c(x)
\ +\ \frac{2\varepsilon\,\bar{\alpha}}{(1-2\varepsilon)\,\alpha_c\,p_c}
\sum_{x\in X_c^*}\! r_c(x)\,\lambda_c(x)\\[2pt]
&\hspace{2.8cm}
\ +\ \frac{C(\varepsilon)}{\alpha_c\,p_c}\,
\sum_{x\in X_c^*}\! t_c(x)\,\lambda_c(x).
\end{aligned}
\]
Since $r_c(x)=\dfrac{P_0(x)}{\lambda_c(x)}$ and $t_c(x)=\dfrac{\Sctr(x)}{\lambda_c(x)}$, it holds
\[
\sum_{x\in X_c^*}\! r_c(x)\,\lambda_c(x)\ =\ \sum_{x\in X_c^*}\! P_0(x)\ =:\ P_0(X_c^*)\ \le\ \xi_c,
\]
and
\[
\sum_{x\in X_c^*}\! t_c(x)\,\lambda_c(x)\ =\ \sum_{x\in X_c^*}\! \Sctr(x)
\ =\ \sum_{j:\,y_j^*\neq y_c^*}\alpha_j \sum_{x\in X_c^*}\!P_j^X(x)
\ \le\ \xi\,\big(\bar{\alpha}-\alpha_c-\beta_c\big),
\]
where the last inequality uses global $\xi$–uniqueness assumption.

Combining the displays yields
\[
S_c(\alpha_c)\ \ge\ 1\;-\;\xi_c \cdot
\frac{(1-\bar{\alpha})\,B_{y_c^*}(\varepsilon,p_c)\;+\;\displaystyle\frac{2\varepsilon\,\bar{\alpha}}{1-2\varepsilon}}{\alpha_c\,p_c}
\;-\;\xi \cdot C(\varepsilon)\cdot
\frac{\bar{\alpha}-\alpha_c-\beta_c}{\alpha_c\,p_c}.
\]
Recalling $B_{y_c^*}(\varepsilon,p_c)=\dfrac{1+2\varepsilon-2p_c}{1-2\varepsilon}$ and $C(\varepsilon)=\dfrac{1+2\varepsilon}{1-2\varepsilon}$, we obtain
\[
\;
S_c(\alpha_c)\ \ge\ 1\;-\;
\xi_c\,\frac{(1-\bar{\alpha})(1+2\varepsilon-2p_c)\;+\;2\varepsilon\,\bar{\alpha}}{(1-2\varepsilon)\,\alpha_c\,p_c}
\;-\;
\xi\,\frac{1+2\varepsilon}{1-2\varepsilon}\,\frac{\bar{\alpha}-\alpha_c-\beta_c}{\alpha_c\,p_c}\;
\]

\end{proof}
\end{document}